\documentclass[sigconf]{acmart}

\usepackage{multicol}
\usepackage{xcolor}
\usepackage{hyperref}
\hypersetup{
    colorlinks=false,
    bookmarks=true,
}
\usepackage{amsmath} 
\usepackage{amsfonts}
\usepackage{graphicx} 
\usepackage{tabularx}
\usepackage{blkarray}
\usepackage{physics} % for different vector notations
\usepackage{caption}
\usepackage{subcaption}
\usepackage{stackrel}
\usepackage{mathtools}
\usepackage{multirow}
\usepackage{algorithmic}
\usepackage{algorithm}
\usepackage{booktabs}
\usepackage{thm-restate}

\newcommand{\startpara}[1]{{\vskip1pt\noindent{\bf #1.}}} % custom paragraph header

\newcommand{\sectref}[1]{Section~\ref{#1}}

\newcommand{\figref}[1]{Figure~\ref{#1}}
\newcommand{\tabref}[1]{Table~\ref{#1}}
\newcommand{\egref}[1]{Example~\ref{#1}}

\newcommand{\thmref}[1]{Theorem~\ref{#1}}

\newcommand{\lemref}[1]{Lemma~\ref{#1}}

\newcommand{\agref}[1]{Algorithm~\ref{#1}}

\renewcommand{\url}[1]{{\def~{\char126}\sf#1}}
\def\Rset{\mathbb{R}}

\def\cM{{\mathcal{M}}}

\def\Dist{{\mathit{Dist}}}
\def\IPaths{{\mathit{IPaths}}}
\def\FPaths{{\mathit{FPaths}}}

\def\Prms{{\mathit{Pr}_{\cM}^{\sigma}}}

\def\Ems{{\mathit{E}_{\cM}^{\sigma}}}
\def\Emss{{\mathit{E}_{\cM,s}^{\sigma}}}
\def\Emsi{{\mathit{E}_{\cM,s_0}^{\sigma}}}
\def\Emso{{\mathit{E}_{\cM}^{\sigma^*}}}

\DeclareMathOperator*{\minimize}{minimize}

\def\rmdef{\stackrel{\mbox{\rm {\tiny def}}}{=}}
\renewcommand\footnotetextcopyrightpermission[1]{}

\begin{document}

\title[Multi-Objective Controller Synthesis with Uncertain Human Preferences]
{Multi-Objective Controller Synthesis \\ with Uncertain Human Preferences}


\author{Shenghui Chen}
\authornote{Equal contribution. This research was conducted when Shenghui Chen was a student at the University of Virginia.}
\affiliation{University of Texas at Austin \country{USA}}
\email{shenghui.chen@utexas.edu}

\author{Kayla Boggess}
\authornotemark[1]
\affiliation{University of Virginia \country{USA}}
\email{kjb5we@virginia.edu}

\author{David Parker}
\affiliation{University of Birmingham \country{UK}}
\email{d.a.parker@cs.bham.ac.uk}

\author{Lu Feng}
\affiliation{University of Virginia \country{USA}}
\email{lu.feng@virginia.edu}


\begin{abstract}
 
Complex real-world applications of cyber-physical systems give rise to the need for \emph{multi-objective controller synthesis}, which concerns the problem of computing an optimal controller subject to multiple (possibly conflicting) criteria. The relative importance of objectives is often specified by human decision-makers. However, there is inherent uncertainty in human preferences (e.g., due to artifacts resulting from different preference elicitation methods). In this paper, we formalize the notion of \emph{uncertain human preferences}, and present a novel approach that accounts for this uncertainty in the context of multi-objective controller synthesis for Markov decision processes (MDPs). Our approach is based on mixed-integer linear programming and synthesizes an optimally permissive multi-strategy that satisfies uncertain human preferences with respect to a multi-objective property. Experimental results on a range of large case studies show that the proposed approach is feasible and scalable across varying MDP model sizes and uncertainty levels of human preferences. Evaluation via an online user study also demonstrates the quality and benefits of the synthesized controllers.    
\end{abstract}


\maketitle
\pagestyle{empty}

\section{Introduction}\label{sec:intro}

%\red{what is controller synthesis, why multi-objective, why human preferences matter }

%\red{controller synthesis for CPS, cite applications}

Controller synthesis---which offers automated techniques to synthesize controllers that satisfy certain properties--- has been increasingly used in the design of cyber-physical systems (CPS), including applications such as semi-autonomous driving~\cite{seshia2016design}, robotic planning~\cite{kress2018synthesis}, and human-in-the-loop CPS control~\cite{feng2015controller}.
Many complex real-world CPS applications give rise to the need for \emph{multi-objective controller synthesis}, which computes an optimal controller subject to multiple (possibly conflicting) criteria.
Examples are synthesizing an optimal controller to maximize safety while minimizing fuel consumption for an automotive vehicle, or synthesizing an optimal robotic controller to minimize the mission completion time while minimizing the risk in disaster search and rescue.
An optimal solution to multi-objective controller synthesis should account for the trade-off between multiple objective properties.
There may not exist a single global solution that optimizes each individual objective property simultaneously.
Instead, a set of \emph{Pareto optimal} points can be computed:
those for which no objective can be optimized further without worsening some other objectives.
%each of which corresponds to the optimal solution of a weighted sum of objectives~\cite{marler2004survey}. 

For many applications that involve human decision-makers, they can be presented with these Pareto optimal solutions to decide which one to choose. 
Alternatively, humans can specify \emph{a priori} their preferences about the relative importance of objectives, which are then used as weights in the multi-objective controller synthesis to compute an optimal solution based on the weighted sum of objectives. 
We can ask humans to assign objective weights directly; however, sometimes it can be difficult for them to come up with these values. 
As surveyed in \cite{marler2004survey}, there exist many different approaches for eliciting human preferences, such as ranking, rating, and pairwise comparison.
Various preference elicitation methods can yield different weight values as artifacts.
Moreover, human preferences can evolve over time and vary across multiple users. 
Thus, there is inherent uncertainty in human preferences. 

In this work, we study the problem of multi-objective controller synthesis with uncertain human preferences.
\emph{To the best of our knowledge, this is the first work that takes into account the uncertainty of human preferences in multi-objective controller synthesis.}
We address the following research challenges:
How to mathematically represent the uncertainty in human preferences?
How to account for uncertain human preferences in multi-objective controller synthesis?
How to generate a succinct representation of the synthesis results? And how to evaluate the synthesized controllers?

Specifically, we focus on the modeling formalism of Markov decision processes (MDPs), which have been popularly applied for the controller synthesis of CPS that exhibit stochastic and nondeterministic behavior (e.g., robots~\cite{kress2018synthesis}, human-in-the-loop CPS~\cite{feng2015controller}).
In recent years, theories and algorithms have been developed for the formal verification and controller synthesis of MDPs subject to multi-objective properties~\cite{chatterjee2006markov,etessami2007multi,forejt2011quantitative,forejt2012pareto,HJKQ20,DKQR20}.
However, none of the existing work takes into account the uncertainty in human preferences.

We formalize the notion of \emph{uncertain human preferences} as an interval weight vector that comprises a convex set of weight vectors over objectives.
Since each weight vector corresponds to some controller that optimizes the weighted sum of objectives, an interval weight vector would yield a set of controllers (i.e., MDP strategies). 
We adopt the notion of \emph{multi-strategy}~\cite{DFK+15} to succinctly represent a set of MDP strategies.  
A (deterministic, memoryless) multi-strategy specifies multiple possible actions in each MDP state. Thus, a multi-strategy represents a set of compliant MDP strategies, each of which chooses an action that is allowed by the multi-strategy in each MDP state. 
We define the soundness of a multi-strategy with respect to a multi-objective property, and an interval weight vector representing uncertain human preferences. 
We also quantify the permissivity of a multi-strategy by measuring the degree to which actions are allowed in (reachable) MDP states.
A sound, permissive multi-strategy can enable more flexibility in CPS design and execution. 
For example, if an action in an MDP state becomes infeasible during the system execution (e.g., some robotic action cannot be executed due to an evolving and uncertain environment), then alternative actions allowed by the multi-strategy can be executed instead, still guaranteeing satisfaction of the human preferences.

We develop a mixed-integer linear programming (MILP) based approach to synthesize a sound, optimally permissive multi-strategy with respect to a multi-objective MDP property and uncertain human preferences. 
Our approach is inspired by~\cite{DFK+15}, which presents an MILP-based method for synthesizing permissive strategies in stochastic games (of which MDPs are a special case).
However, there are several key differences in our encodings. 
First, we solve multi-objective optimization problems, while~\cite{DFK+15} is for a single objective.
Second, we have a different soundness definition for the multi-strategy and need to track the values of both lower and upper bounds of each objective, while~\cite{DFK+15} only considers one direction.
Lastly, we have a different definition of permissivity which only considers reachable states under a multi-strategy.

%, while~\cite{DFK+15} does not distinguish reachable and unreachable states for the MILP encoding. 
%\dave{I suggest to cut from ``while...'' in the previous sentence, since the dynamic penalties in~\cite{DFK+15} can be seen as another way to include reachability}

We evaluate the proposed approach on a range of large case studies. 
The experimental results show that our MILP-based approach is scalable to 
synthesize sound, optimally permissive multi-strategies for large models with more than $10^6$ MDP states.
Moreover, the results show that increasing the uncertainty of human preferences yields 
more permissive multi-strategies.

In addition, we evaluate the quality of synthesized controllers via an online user study with 100 participants using Amazon Mechanical Turk.
The study results show that strategies synthesized based on human preferences are more favorable, perceived as more accurate, and lead to better user satisfaction, compared to arbitrary strategies.
In addition, multi-strategies are perceived as more informative and satisfying than less permissive (single) strategies. 

\vspace*{0.5em}
\startpara{Contributions}
We summarize the major contributions of this work as follows.

\begin{itemize}
    \item We formalized the notion of uncertain human preferences, and developed an MILP-based approach to synthesize a sound, optimally permissive multi-strategy for a given multi-objective MDP property and uncertain human preferences.
    \item We implemented the proposed approach and evaluated it on a range of large case studies to demonstrate its feasibility and scalability. 
    \item We designed and conducted an online user study to evaluate the quality and benefits of the synthesized controllers. 
\end{itemize}

\startpara{Paper Organization}
In the rest of the paper, 
we introduce some background about MDPs and multi-objective properties in \sectref{sec:background},
formalize uncertain human preferences in \sectref{sec:preference},
develop the controller synthesis approach in \sectref{sec:approach},
present experimental results in \sectref{sec:exp},
describe the user study in \sectref{sec:study},
survey related work in \sectref{sec:related}, 
and draw conclusions in \sectref{sec:conclusion}.

\section{Background}\label{sec:background}

% Let $\Rset$ denote the reals. A discrete probability \emph{distribution} over a (countable) set $Q$ is a function $\mu : Q \to [0, 1]$ such that $\sum_{q \in Q} \mu(q) =  1$. Let $\Dist(Q)$ denote the set of distributions over $Q$.

In this section, we introduce the necessary background about MDPs and multi-objective properties. 

A \emph{Markov decision process (MDP)} is a tuple $\cM=(S, s_0, A, \delta)$, 
where $S$ is a finite set of states, 
$s_0\in S$ is an initial state, 
$A$ is a set of actions,
and $\delta: S\times A \to \Dist(S)$ is a probabilistic transition function
with $\Dist(S)$ denoting the set of probability distributions over $S$.
Each state $s\in S$ has a set of \emph{enabled} actions, given by
$\alpha(s) \rmdef \{a\in A | \delta(s,a)\mbox{ is defined}\}$.
A \emph{path} through $\cM$ is a sequence $\pi=s_0a_0s_1a_1\dots$ 
where $a_i \in \alpha(s_i)$ and $\delta(s_i,a_i)(s_{i+1})>0$ for all $i\ge 0$.
We say that a state $s$ is \emph{reachable} if there exists a finite path starting from $s_0$ and ending in $s$ as the last state. 
Let $\FPaths$ ($\IPaths$) denote the set of finite (infinite) paths through $\cM$.

A \emph{strategy} (also called a policy) is a function $\sigma:\FPaths \to \Dist(A)$ that resolves the nondeterministic choice of actions in each state based on the execution history. 
A strategy $\sigma$ is \emph{deterministic} if $\sigma(\pi)$ is a point distribution for all $\pi$, 
and \emph{randomized} otherwise. 
A strategy $\sigma$ is \emph{memoryless} if the action choice $\sigma(\pi)$ depends only on the last state of $\pi$. 
In this work, we focus on deterministic, memoryless strategies.%
\footnote{For the types of MDP properties considered in this work, there always exists a deterministic, memoryless strategy in the solution set~\cite{puterman1994markov,forejt2011quantitative,forejt2012pareto}.}
Thus, we can simplify the definition of strategy to a function $\sigma: S \to A$. 
Let $\Sigma_\cM$ denote the set of all (deterministic, memoryless) strategies for $\cM$.
A strategy $\sigma \in \Sigma_\cM$ induces a probability measure over $\IPaths$, denoted by $\Prms$, in the standard fashion~\cite{kemeny2012denumerable}.

A \emph{reward function} of $\cM$ takes the form $r: S \times A \to \Rset$. 
The \emph{total reward} along an infinite path $\pi=s_0a_0s_1a_1\dots$ is given by 
$r(\pi) \rmdef \sum_{t=0}^{\infty} r(s_t,a_t)$.
The \emph{expected total reward} for $\cM$ under a strategy $\sigma$ is denoted by
$\Ems(r) \rmdef \int_\pi r(\pi) \ d \Prms$.
We say that $\cM$ under strategy $\sigma$ satisfies a \emph{reward predicate} $[r]_{\sim b}$ where $\sim\,\in \{\ge, \le\}$ is a relational operator and $b$ is a rational reward bound, denoted $\cM, \sigma \models [r]_{\sim b}$, 
if the expected total reward $\Ems(r) \sim b$. 
A reward predicate $[r]_{\sim b}$ is \emph{satisfiable} in MDP $\cM$
if there exists a strategy $\sigma \in \Sigma_\cM$ such that $\cM, \sigma \models [r]_{\sim b}$.
If $b$ is unspecified, we can ask numerical queries, denoted
$[r]_{\min} \rmdef \inf \{x \in \Rset \ |\ [r]_{\le x} \mbox{ is satisfiable}\}$
and $[r]_{\max} \rmdef \sup \{x \in \Rset \ |\ [r]_{\ge x} \mbox{ is satisfiable}\}$.

A \emph{multi-objective property} $\phi=([r_1]_{\bowtie_1}$, $\dots,[r_n]_{\bowtie_n})$, 
where $\bowtie_i\,\in \{\min,\max\}$,
aims to minimize and/or maximize $n$ objectives of expected total rewards simultaneously.
For the rest of the paper, we assume that the multi-objective property 
is of the form $\phi=([r_1]_{\min}, \dots,[r_n]_{\min})$. 
A maximizing objective $[r_i]_{\max}$ can be converted to a minimizing objective by negating rewards. 
Checking $\phi$ on MDP $\cM$ yields a set of Pareto optimal points that lie on the boundary of the set of \emph{achievable values}: 
$$X = \{\vb*{x}=(x_1,\dots,x_n) \in \Rset^n \ |\ ([r_1]_{\le x_1}, \dots,[r_n]_{\le x_n}) \mbox{ is satisfiable}\}.$$
We say that a point $\vb*{x^*}=(x^*_1,\dots,x^*_n) \in X$ is \emph{Pareto optimal} if there does not exist another point $\vb*{x}=(x_1,\dots,x_n) \in X$ such that $x_i \le x^*_i$ for all $i$ and $x_j \neq x^*_j$ for some $j$.
A multi-objective reward predicate $([r_1]_{\le x_1}, \dots,[r_n]_{\le x_n})$ is satisfiable in MDP $\cM$ if there exists a strategy $\sigma \in \Sigma_\cM$ such that $\cM, \sigma \models [r_i]_{\le x_i}$ for all $i$.
The set of achievable values $X$ for $\phi$ is convex~\cite{forejt2012pareto}.

Given a multi-objective property $\phi=([r_1]_{\min}, \dots,[r_n]_{\min})$ 
and a weight vector $\vb*{w} \in \Rset^n$, 
the \emph{expected total weighted reward sum} is 
$\Ems(\vb*{w} \cdot \vb*{r}) \rmdef \sum_{i=1}^n w_i \Ems(r_i)$
for any strategy $\sigma \in \Sigma_\cM$.
We say that a strategy $\sigma^*$ is \emph{optimal} with respect to $\phi$ and $\vb*{w}$, if 
$\Emso(\vb*{w} \cdot \vb*{r}) = \inf \{\Ems(\vb*{w} \cdot \vb*{r}) \ |\ \sigma \in \Sigma_\cM  \}$.
The strategy $\sigma^*$ also corresponds to a Pareto optimal point for $\phi$~\cite{forejt2012pareto}.

For simplicity, in this paper, we make the assumption that an MDP has a set of \emph{end states},
which are reached with probability 1 under any strategy,
and have zero reward and no outgoing transitions to other states.
This simplifies our analysis by ensuring that the expected total reward is always finite.
A variety of useful objectives for real-world applications can be encoded under these restrictions,
for example, minimizing the distance, time, or incurred risk to complete a navigation task for robotic planning, or maximizing the safety and driver trust to complete a trip for autonomous driving.

\begin{figure}[tb]
    \centering
    \includegraphics[width=1\columnwidth]{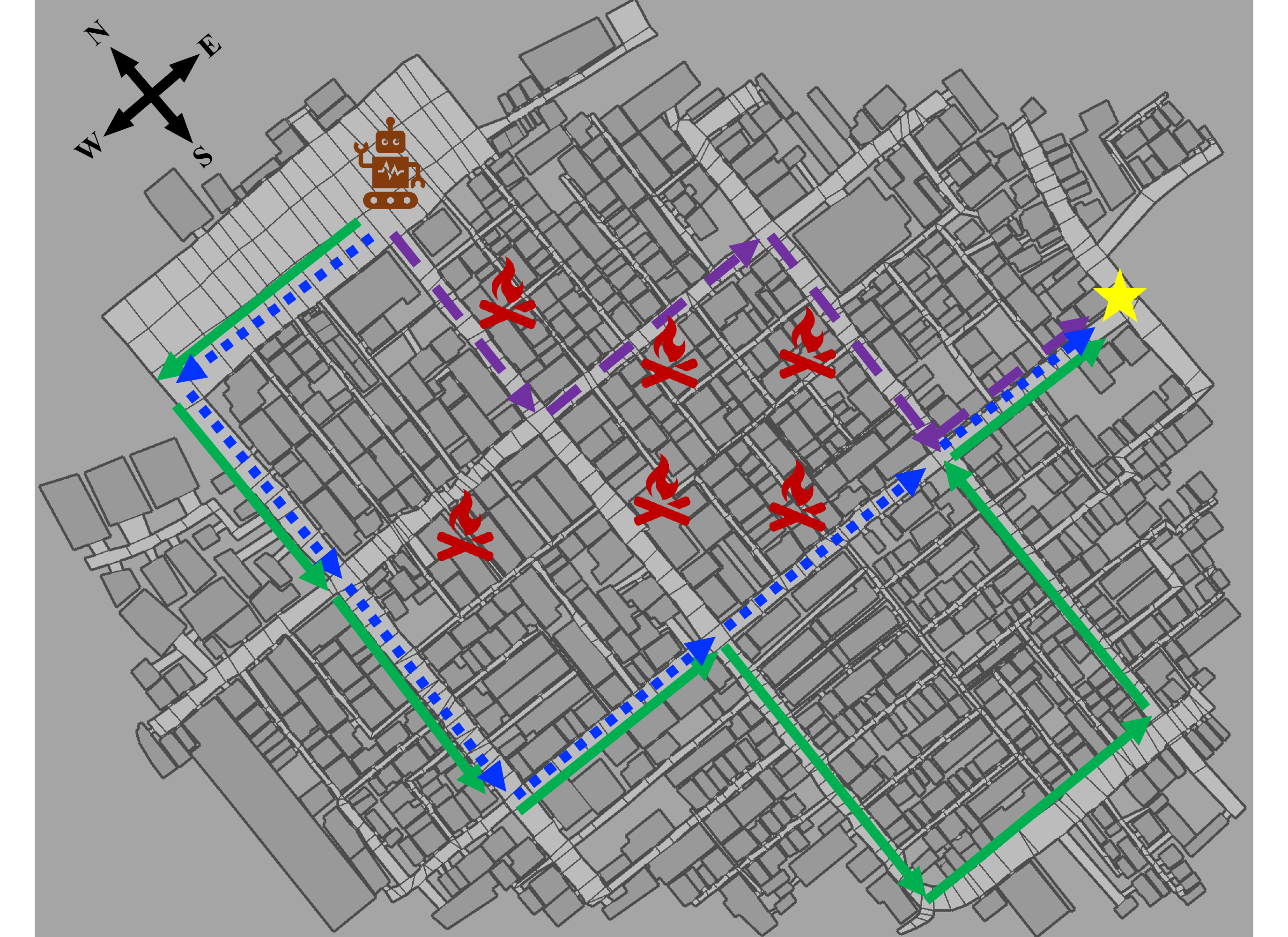}
    \caption{An example map for robotic planning in urban search and rescue missions. The robot aims to navigate to the victim (star) location with the shortest distance while minimizing the risk of bypassing (red) fire zones.}
    \label{fig:map}
\end{figure}

\begin{figure}[tb]
    \centering
    \includegraphics[width=1\columnwidth]{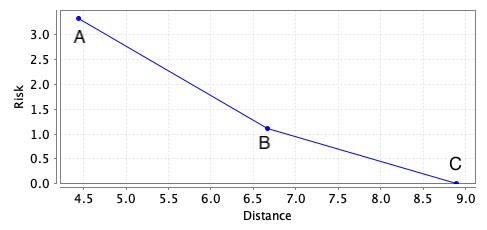}
    \caption{Pareto curve for multi-objective robotic planning. Purple, blue, and green routes in \figref{fig:map} correspond to Pareto optimal points A, B and C, respectively.}
    \label{fig:pareto}
\end{figure}

\begin{example}\label{eg:map}
\figref{fig:map} shows a map for urban search and rescue missions taken from the RoboCup Rescue Simulation Competition~\cite{akin2013robocup}. 
Consider a scenario where the robot aims to find an optimal route satisfying two objectives: 
(1) minimizing the travel distance to reach the rescue location, and (2) minimizing the risk of bypassing fire zones. 
We model the problem as an MDP where each road junction in the map is represented by an MDP state. 
In each state, the robot can move along the road with probability 0.9 and stay put with probability 0.1 due to noisy sensors.
We define two reward functions $\mathsf{dist}$ and $\mathsf{risk}$ to measure the distance (i.e., the number of road blocks navigated) and the risk (i.e., the number of fire zones bypassed), respectively. 
\figref{fig:pareto} shows the Pareto curve for the multi-objective property 
$\phi = ([\mathsf{dist}]_{\min}, [\mathsf{risk}]_{\min})$.
The convex set of achievable values for $\phi$ includes any point on the Pareto curve and in the area above. 
There are three Pareto optimal points (A, B, C) corresponding to three deterministic, memoryless MDP strategies illustrated as purple, blue, and green routes in \figref{fig:map}, respectively. 
The rest of the Pareto curve (e.g., any point on the solid line between A and B, or the solid line between B and C) is achievable only if the robot takes randomized strategies. 
\end{example}

\section{Uncertain Human Preferences}\label{sec:preference}

%\red{why uncertainty? how to derive weights from preferences? how to formalize uncertain human preferences? what's the implication on controller synthesis? }

\subsection{Formalization of Preferences}

Preferences are often represented as weights reflecting humans' opinions about the relative importance of different criteria in multi-objective optimization~\cite{marler2004survey}.
Following this convention, we denote a \emph{preference} over $n$ objectives as a weight vector
$\vb*{w}=(w_1, \dots, w_n) \in \Rset^n$ where $w_i \ge 0$ for $1 \le i \le n$ and $\sum_{i=1}^n w_i=1$.
%\blue{(Is it necessary to restrict to $w_i > 0$ to guarantee Pareto optimality? See page 374-375 of \cite{marler2004survey})}

Such weight vectors can be obtained by eliciting human preferences in different ways.
A naive approach is to ask for direct human input of weight values for objectives; however, it may be difficult for humans to come up with these values in practice. 
A popular preference elicitation method is pairwise comparison~\cite{thurstone1927law}, in which humans answer queries such as: ``Do you prefer objective $i$ or objective $j$?'' for each pair of objectives.
We can then derive weights (e.g., via finding eigenvalues of pairwise comparison matrices) as described in~\cite{barzilai1997deriving,dijkstra2013extraction}.
There are many other methods (e.g., Likert scaling, rating, ranking) for eliciting preferences weights, as surveyed in~\cite{marler2004survey}.
Eliciting preferences from the same person using various methods can yield different weight vectors as artifacts.
In addition, if the controller synthesis needs to account for multiple human decision-makers' opinions, then a range of weight vectors can be resulted from eliciting multiple humans' preferences.

In order to capture the inherent uncertainty of human preferences, 
we define \emph{uncertain human preferences} as an interval weight vector
$\vb*{\tilde{w}}=([\underline{w}_1, \overline{w}_1], \cdots, [\underline{w}_n, \overline{w}_n])$,
where $\underline{w}_i$ ($\overline{w}_i$) is the lower (upper) weight bound for objective $i$, and $0 \le \underline{w}_i \le \overline{w}_i \le 1$.
We say that a weight vector $\vb*{w}$ belongs to 
an interval weight vector $\vb*{\tilde{w}}$, denoted $\vb*{w} \in \vb*{\tilde{w}}$,
if $\underline{w}_i \le w_i \le \overline{w}_i$ for all $i$.
An interval weight vector comprises a convex set of weight vectors, providing a compact representation of uncertain human preferences. 

\begin{figure}[tb]
    \centering
    \includegraphics[width=0.6\columnwidth]{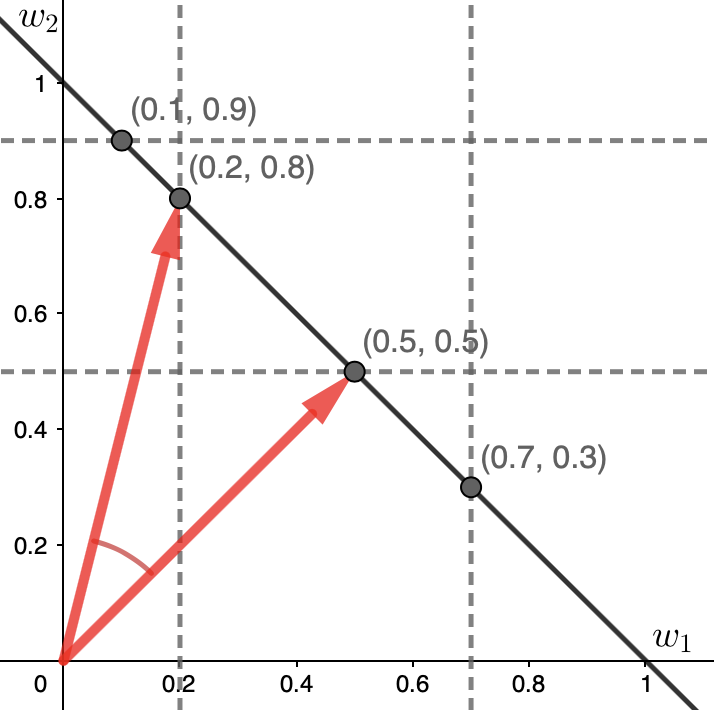}
    \caption{Geometrical interpretation of an interval weight vector
    $\vb*{\tilde{w}}=([0.2, 0.7], [0.5, 0.9])$ representing uncertain human preferences.}
    \label{fig:weights}
\end{figure}

\begin{example}\label{eg:weight}
Suppose $\vb*{\tilde{w}}=([0.2, 0.7], [0.5, 0.9])$.
\figref{fig:weights} shows a geometrical interpretation of uncertain human preferences represented by $\vb*{\tilde{w}}$.
We intersect each dashed line representing the lower or upper objective bounds with the solid line representing $w_1 + w_2 = 1$, and obtain a pair of weight vectors $(0.2, 0.8)$ and $(0.5, 0.5)$ corresponding to the extreme points of the feasible solution set. 
We highlight in red the range of all possible weight vectors that belong to $\vb*{\tilde{w}}$, representing uncertain human preferences. 
\end{example}

\subsection{Multi-Strategy for MDPs}

Recall from \sectref{sec:background} that an optimal MDP strategy $\sigma^*$ with respect to a multi-objective property $\phi$ and a weight vector $\vb*{w}$ corresponds to a Pareto optimal point that optimizes the weighted sum of objectives.
Thus, an interval weight vector $\vb*{\tilde{w}}$ representing uncertain human preferences yields a set of Pareto optimal points and corresponding MDP strategies. 
We will use the notion of \emph{multi-strategy} from \emph{permissive controller synthesis}~\cite{DFK+15} to succinctly represent a set of strategies as follows.

A (deterministic, memoryless) multi-strategy for MDP $\cM$ is 
a function $\theta: S \to 2^A$,
defining a set of \emph{allowed} actions $\theta(s) \subseteq \alpha(s)$ in each state $s \in S$.
Let $\Theta_\cM$ denote the set of all multi-strategies for $\cM$.
We say that a (deterministic, memoryless) strategy $\sigma$ \emph{complies} with multi-strategy $\theta$, denoted $\sigma \lhd \theta$,
if $\sigma(s) \in \theta(s)$ for all states $s\in S$.
We require that $\theta(s)\neq\emptyset$ for any state $s$ that is reachable
under some strategy that complies with $\theta$.

%\dave{I reformulated the non-emptiness condition above. I also removed $\beta$, which should not be needed if we restrict to deterministic multi-strategies}

Given a reward predicate $[r]_{\sim b}$,
we say that multi-strategy $\theta$ is \emph{sound} with respect to $[r]_{\sim b}$
if $\cM, \sigma \models [r]_{\sim b}$ for every strategy $\sigma$ that complies with $\theta$.
We then say that a multi-strategy is sound for an uncertain set of human preferences
if it is sound with respect to upper and lower bounds on each
objective induced by a set of weight intervals.
More precisely, given a multi-objective property $\phi=([r_1]_{\min}, \dots,[r_n]_{\min})$ and an interval weight vector $\vb*{\tilde{w}}$,
we say that multi-strategy $\theta$ is sound with respect to $\phi,\vb*{\tilde{w}}$
if it is sound with respect to $[r_i]_{\ge \underline{b}_i}$ and $[r_i]_{\le \overline{b}_i}$ for all $i$,
where $\underline{b}_i = \inf\{x_i | \vb*{x}=(x_1,\dots,x_n) \in X_{\vb*{\tilde{w}}} \}$, $\overline{b}_i = \sup\{x_i | \vb*{x}=(x_1,\dots,x_n) \in X_{\vb*{\tilde{w}}} \}$, and $X_{\vb*{\tilde{w}}}$ denotes the set of Pareto optimal points corresponding to $\vb*{\tilde{w}}$. 
The intuition is that, due to convexity, any weight vector $\vb*{w} \in \vb*{\tilde{w}}$ must 
correspond to a Pareto optimal point within a space bounded by extreme points of $X_{\vb*{\tilde{w}}}$.
Later we develop \agref{alg:bounds} in \sectref{sec:approach} to compute values of $\underline{b}_i$ and $\overline{b}_i$.

We quantify the \emph{permissivity} of multi-strategy $\theta$ by measuring the degree of actions allowed in (reachable) MDP states. 
Let $\lambda(\theta) \rmdef \sum_{s \in S^\theta} (|\alpha(s)|-|\theta(s)|)$ be a \emph{penalty function} 
where $S^\theta \subseteq S$ is the set of reachable states under $\theta$.
We say that a sound multi-strategy $\theta^*$ for $\cM$ is \emph{optimally permissive} if 
$\lambda(\theta^*) = \inf \{\lambda(\theta) \ |\ \theta \in \Theta_\cM$ is sound with respect to 
$\phi \mbox{ and } \vb*{\tilde{w}}\}$. 

% Later in \sectref{sec:exp}, we use a normalized value $1-\frac{\lambda(\theta)}{\sum_{s \in S^\theta} |\alpha(s)|}$ to compare the permissivity among different case studies. 

\section{Controller Synthesis Approach}\label{sec:approach}

\subsection{Problem Statement}\label{sec:problem}
Given an MDP $\cM=(S, s_0, A, \delta)$, a multi-objective property $\phi=([r_1]_{\min}, \dots,[r_n]_{\min})$, and an interval weight vector $\vb*{\tilde{w}}$ representing uncertain human preferences, how can we synthesize an optimally permissive multi-strategy $\theta \in \Theta_\cM$ that is sound with respect to $\phi$ and $\vb*{\tilde{w}}$?

\subsection{MILP-based Solution}
We present a mixed-integer linear programming (MILP) based approach to solve the above problem.
We use binary variables $\eta_{s,a} \in \{0,1\}$ to encode whether a multi-strategy $\theta$ allows action $a \in \alpha(s)$ in state $s \in S$ of MDP $\cM$.
We use real-valued variables $\mu_{i,s}$ and $\nu_{i,s}$ to represent the minimal and maximal expected total reward for the $i$th objective from state $s$, under any strategy complying with $\theta$. 
We set $\mu_{i,s}=\nu_{i,s}=0$ for any end states in the MDP.
The MILP encoding is:

%-------------------------------------------------------------
\begin{subequations}
\begin{align}
%-----------------------
\minimize_{\eta_{s,a} \in \{0,1\}, \mu_{i,s} \in \Rset, \nu_{i,s} \in \Rset} 
c \cdot\sum_{s\in S} \sum_{a\in \alpha(s)} (1-\eta_{s,a}) \nonumber \\
+ \sum_{i=1}^n (\nu_{i,s_0}-\mu_{i,s_0})
\label{eq:obj} \\
\intertext{\hspace{60pt} subject to}
%-----------------------
 \forall s\in S: \sum_{a\in \alpha(s)} \eta_{s,a} \le c \cdot \sum_{(t,a) \in \rho(s)} \eta_{t,a} 
 \label{eq:c1} \\
%-----------------------
 \forall s\in S: c \cdot \sum_{a\in \alpha(s)} \eta_{s,a} \ge \sum_{(t,a) \in \rho(s)} \eta_{t,a} 
 \label{eq:c2} \\
%-----------------------
\forall 1\leq i\leq n,\forall s \in S, \forall a\in \alpha(s): \nonumber \\
 \mu_{i,s} \le \sum_{t\in S} \delta(s,a)(t) \cdot \mu_{i,t} + r_i(s,a) + c \cdot (1-\eta_{s,a})
 \label{eq:c3} \\
%-----------------------
\forall 1\leq i\leq n, \forall s \in S, \forall a\in \alpha(s): \nonumber \\
 \nu_{i,s} \ge \sum_{t\in S} \delta(s,a)(t) \cdot \nu_{i,t} + r_i(s,a) - c \cdot (1-\eta_{s,a})
 \label{eq:c4} \\
%-----------------------
 \forall 1\leq i\leq n: \mu_{i,s_0}  \ge \underline{b}_i
 \label{eq:c5} \\
%-----------------------
 \forall 1\leq i\leq n: \nu_{i,s_0}  \le \overline{b}_i
 \label{eq:c6} 
 %-----------------------
\end{align}
\end{subequations}
%-------------------------------------------------------------
where $c$ is a large scaling constant%
\footnote{Constant $c$ is chosen to be larger than the expected total reward for any objective, from any state and under any objective.}
and we let $\rho(s) \rmdef \{(t,a) \ |\ $
$\delta(t,a)(s)>0 \mbox{ and } t\neq s\}$
denote the set of incoming transitions to a state $s \in S$. 

The objective function (\ref{eq:obj}) minimizes the total number of disallowed actions in all states
%(with a large constant $c$ as the scaling factor)
plus the sum of expected total rewards over all objectives in the initial state. 
The latter serves as a tie-breaker between solutions with the same permissivity, 
favoring tighter reward bounds. 

Constraints (\ref{eq:c1}) and (\ref{eq:c2}) enforce that no action is allowed for $s$ if it is unreachable from any other state under the multi-strategy, and at least one action should be allowed otherwise.
For the initial state $s_0$, we assume that there is always an allowed incoming transition, 
and $\sum_{(t,a) \in \rho(s_0)} \eta_{t,a}=1$.
Constraints (\ref{eq:c3}) and (\ref{eq:c4}) encode the recursion for expected rewards in each step,
which are trivially satisfied when $\eta_{s,a}=0$, that is, action $a$ is disallowed in state $s$. 
Constraints (\ref{eq:c5}) and (\ref{eq:c6}) guarantee that, for each objective $i$,
the expected total reward in the initial state under the multi-strategy
satisfies the lower and upper bounds $\underline{b}_i$ and $\overline{b}_i$, 
which are precomputed using \agref{alg:bounds}.

%------------------------------------------------------------
\begin{algorithm}[t] 
\caption{Precomputing objective bounds} \label{alg:bounds} 
\algsetup{linenosize=\tiny}
\begin{algorithmic} [1] 
    \REQUIRE An MDP $\cM$, a multi-objective property $\phi$, and an interval weight vector $\vb*{\tilde{w}}$ for uncertain human preferences
    \ENSURE An interval vector $\vb*{b}=([\underline{b}_1, \overline{b}_1], \dots, [\underline{b}_n, \overline{b}_n])$ for expected total reward bounds over $n$ objectives
	\STATE Initialize $\underline{b}_i=\infty$ and $\overline{b}_i=-\infty$ for $1 \le i \le n$
	\STATE $W \leftarrow$ Find the set of extreme points of $\vb*{\tilde{w}}$ 
	\FORALL{weight vector $\vb*{w} \in W$ } 
	    \STATE $\vb*{x}=(x_1,\dots,x_n) \leftarrow$ 
	        Find a Pareto optimal point for $\phi$ that corresponds to $\vb*{w}$
	    \FOR{$1 \le i \le n$}
	        \IF{$\underline{b}_i \ge x_i$}
	            \STATE $\underline{b}_i = x_i$
	        \ENDIF
	        \IF{$\overline{b}_i \le x_i$}
	            \STATE $\overline{b}_i = x_i$
	        \ENDIF
	    \ENDFOR
	\ENDFOR
    \RETURN $\vb*{b}$
\end{algorithmic}
\end{algorithm}
%------------------------------------------------------------

Given an interval weight vector $\vb*{\tilde{w}}$ representing uncertain human preferences, 
\agref{alg:bounds} (line 2) first finds the set of extreme points in $\vb*{\tilde{w}}$, denoted $W$.
This can be done by applying standard methods for finding extreme points in a convex set~\cite{treves2016topological}.
Next, for each weight vector $\vb*{w} \in W$, \agref{alg:bounds} (line 4) finds a 
Pareto optimal point $\vb*{x}=(x_1,\dots,x_n)$ for $\phi$,
which yields the minimal expected total weighted reward sum under any strategy of the MDP $\cM$. 
Here, we apply the value iteration-based method in~\cite{forejt2012pareto} 
for the computation of Pareto optimal points. 
Finally, \agref{alg:bounds} (line 3-13) loops through all weight vectors in $W$ 
to determine the smallest lower bound $\underline{b}_i$ 
and the greatest upper bound $\overline{b}_i$ of the expected total reward for each objective $i$.
%\red{Where and how to say something about the correctness of \agref{alg:bounds}? Why does it suffice to only consider extreme points of interval weight vectors?}

\begin{example}\label{eg:milp}
We apply the proposed approach to synthesize an optimally permissive multi-strategy for MDP $\cM$ modeled in \egref{eg:map}
that is sound with respect to $\phi = ([\mathsf{dist}]_{\min}, [\mathsf{risk}]_{\min})$ 
and $\vb*{\tilde{w}}=([0.2, 0.7], [0.5, 0.9])$.
Following \egref{eg:weight}, $\vb*{\tilde{w}}$ gives a convex set of weight vectors
with two extreme points $(0.5, 0.5)$ and $(0.2, 0.8)$.
% Geometrically, we can determine the Pareto optimal point corresponding to a weight vector $\vb*{w}$ 
% by moving a perpendicular line of $\vb*{w}$ parallel to itself until it hits the curve~\cite{das1997closer}. 
We also find out that weight vectors $(0.5, 0.5)$ and $(0.2, 0.8)$ correspond to Pareto optimal points B and C in \figref{fig:pareto}, respectively.
Thus, applying \agref{alg:bounds} yields an interval vector 
$\vb*{b}=([6.66, 8.89], [0, 1.12])$ for the expected total reward bounds over $\phi$. 

The MILP encoding minimizes
$c \cdot\sum_{s\in S} \sum_{a\in \alpha(s)} (1-\eta_{s,a}) \nonumber 
+ \sum_{i=1}^2 (\nu_{i,s_0}-\mu_{i,s_0})$.
We can select $c=1000$ as the scaling factor constant in this example.

Constraints (\ref{eq:c1}) and (\ref{eq:c2}) are instantiated, for example, for the initial state $s_0$ as:
$$ \eta_{s_0,\mathsf{south}} + \eta_{s_0,\mathsf{west}} \le c $$
$$ c \cdot (\eta_{s_0,\mathsf{south}} + \eta_{s_0,\mathsf{west}}) \ge 1 $$

Constraints (\ref{eq:c3}) and (\ref{eq:c4}) are instantiated, for example, 
for the first objective $[\mathsf{dist}]_{\min}$, state $s_0$, and action $\mathsf{west}$ as:
$$ \mu_{1,s_0} \le {0.9 \cdot \mu_{1,s_1}} + {0.1 \cdot \mu_{1,s_0}} + 1 + c \cdot (1-\eta_{s_0,\mathsf{west}}) $$
$$ \nu_{1,s_0} \ge {0.9 \cdot \nu_{1,s_1}} + {0.1 \cdot \nu_{1,s_0}} + 1 - c \cdot (1-\eta_{s_0,\mathsf{west}}) $$

Constraints (\ref{eq:c5}) and (\ref{eq:c6}) are instantiated, for example, 
for the first objective $[\mathsf{dist}]_{\min}$ as:
$\mu_{1,s_0}  \ge 6.66$ and $\nu_{1,s_0}  \le 8.89$.

The MILP encoding uses 15 binary variables to encode $\eta_{s,a}$, 
44 real-valued variables to encode $\mu_{i,s}$ and $\nu_{i,s}$,
and a total number of 90 constraints. 
It takes less than 1 second to solve the MILP problem using the Gurobi optimization toolbox~\cite{gurobi}.
The solution yields a multi-strategy as illustrated by the orange lines in \figref{fig:multi}.
The synthesized multi-strategy is sound with respect to $\phi$ and $\vb*{\tilde{w}}$. 
There are two strategies complying with the multi-strategy, 
corresponding to Pareto optimal points B and C in \figref{fig:pareto}.
The multi-strategy is also optimally permissive. 
Such a permissive multi-strategy could be useful in assisting humans' decision-making, 
by informing them about multiple allowable action choices in states. 
In addition, it offers flexibility for the system execution. If the robot finds that certain action cannot be executed due to the evolving environment (e.g., fire spreading), it may execute an alternative actions allowed by the multi-strategy while still guaranteeing soundness.

\end{example}

\begin{figure}[tb]
    \centering
    \includegraphics[width=1\columnwidth]{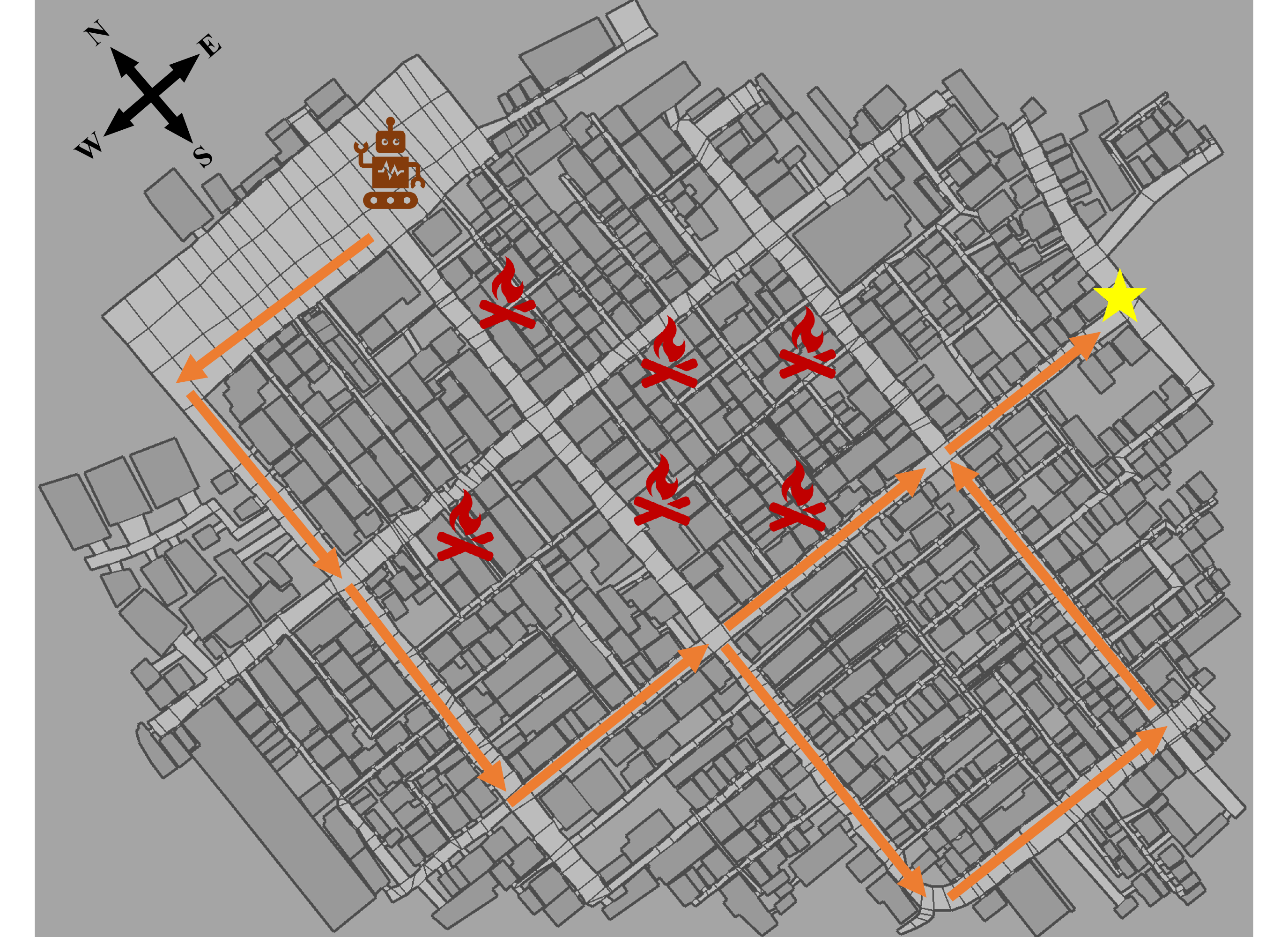}
    \caption{The synthesized multi-strategy in \egref{eg:milp}.}
    \label{fig:multi}
\end{figure}

\subsection{Correctness}
The correctness of our proposed approach, with respect to the problem statement in \sectref{sec:problem},
is stated below and the proof is given in the appendix. 
\begin{restatable}{thm}{correctness}
\label{thm:milp}
Let $\cM$ be an MDP, $\phi=([r_1]_{\min}, \dots,[r_n]_{\min})$ be a multi-objective property
and $\vb*{\tilde{w}}$ be an interval weight vector representing uncertain human preferences.
There is a sound, optimally permissive multi-strategy $\theta$ in $\cM$ with respect to $\phi$ and $\vb*{\tilde{w}}$ 
whose permissive penalty is $\lambda(\theta)$, 
if and only if there is an optimal assignment to the MILP instance from (\ref{eq:obj})-(\ref{eq:c6})
which satisfies $\lambda(\theta)=\sum_{s\in S} \sum_{a\in \alpha(s)} (1-\eta_{s,a})$.
\end{restatable}

\begin{table*}[tb]
\tiny
\centering
\resizebox{\textwidth}{!}{%
\begin{tabular}{@{}cccccccccc@{}}
\toprule
\multicolumn{3}{c}{Case Study} & \multicolumn{2}{c}{MDP Size} & \multicolumn{3}{c}{MILP Size} & \multicolumn{2}{c}{MILP Solution} \\ \cmidrule(lr){1-3} \cmidrule(lr){4-5} \cmidrule(lr){6-8} \cmidrule(lr){9-10} 
Name & Parameters & Preferences & \#States & \#Trans & \#Binary & \#Real & \#Constraints & Time (s) & \#Permissive States \\ \midrule
%=====================
\multirow{8}{*}{uav} 
 & \multirow{2}{*}{5} & ([0.1, 0.2], [0.8, 0.9]) & \multirow{2}{*}{28,401} & \multirow{2}{*}{40,373} & \multirow{2}{*}{29,897} & \multirow{2}{*}{113,604} & \multirow{2}{*}{176,394} & 2.2 & 1 \\ 
 &  & ([0.1, 1], [0, 0.9]) &  &  &  &  &  & 7.8 & 1,496 \\ \cmidrule(l){2-10} 
%---------------------
 & \multirow{2}{*}{10} & ([0.1, 0.2], [0.8, 0.9]) & \multirow{2}{*}{56,901} & \multirow{2}{*}{80,873} & \multirow{2}{*}{59,897} & \multirow{2}{*}{227,604} & \multirow{2}{*}{353,394} & 5.4 & 1 \\ 
 &  & ([0.1, 1], [0, 0.9]) &  &  &  &  &  & 20.2 & 2,996 \\ \cmidrule(l){2-10} 
%---------------------
 & \multirow{2}{*}{20} & ([0.1, 0.2], [0.8, 0.9]) & \multirow{2}{*}{113,901} & \multirow{2}{*}{161,873} & \multirow{2}{*}{119,897} & \multirow{2}{*}{455,604} & \multirow{2}{*}{707,394} & 15.3 & 1 \\ 
 &  & ([0.1, 1], [0, 0.9]) &  &  &  &  &  & 43.8 & 5,996 \\ \midrule
%=====================
\multirow{8}{*}{taskgraph} 
 & \multirow{2}{*}{30} & ([0.8, 1], [0, 0.2]) & \multirow{2}{*}{21,046} & \multirow{2}{*}{43,257} & \multirow{2}{*}{29,813} & \multirow{2}{*}{84,184} & \multirow{2}{*}{161,348} & 25.1 & 1,770 \\ 
 &  & ([0.1, 1], [0, 0.9]) &  &  &  &  &  & 4.6 & 8,765 \\ \cmidrule(l){2-10} 
%---------------------
 & \multirow{2}{*}{40} & ([0.8, 1], [0, 0.2]) & \multirow{2}{*}{36,866} & \multirow{2}{*}{75,677} & \multirow{2}{*}{52,153} & \multirow{2}{*}{147,464} & \multirow{2}{*}{282,348} & 54.1 & 4,387 \\ 
 &  & ([0.1, 1], [0, 0.9]) &  &  &  &  &  & 11.7 & 15,285 \\ \cmidrule(l){2-10} 
%---------------------
 & \multirow{2}{*}{50} & ([0.8, 1], [0, 0.2]) & \multirow{2}{*}{57,086} & \multirow{2}{*}{117,097} & \multirow{2}{*}{80,693} & \multirow{2}{*}{228,344} & \multirow{2}{*}{436,948} & 132.5 & 7,939 \\ 
 &  & ([0.1, 1], [0, 0.9]) &  &  &  &  &  & 14.6 & 23,605 \\ \midrule
%=====================
\multirow{6}{*}{teamform} 
 & \multirow{2}{*}{2} & ([0.8, 1], [0, 0.2]) & \multirow{2}{*}{1,847} & \multirow{2}{*}{2,288} & \multirow{2}{*}{2,191} & \multirow{2}{*}{7,388} & \multirow{2}{*}{12,462} & 1.9 & 146 \\
 &  & ([0, 0.9], [0.1, 1]) &  &  &  &  &  & 2.3 & 189 \\ \cmidrule(l){2-10} 
%---------------------
 & \multirow{2}{*}{3} & ([0.8, 1], [0, 0.2]) & \multirow{2}{*}{12,475} & \multirow{2}{*}{15,228} & \multirow{2}{*}{14,935} & \multirow{2}{*}{49,900} & \multirow{2}{*}{84,694} & timeout & - \\ 
 &  & ([0, 0.9], [0.1, 1]) &  &  &  &  &  & timeout & - \\  \bottomrule
\end{tabular}%
}
\vspace{1pt}
\caption{Experimental results illustrating performance of the proposed approach}
\label{tab:results}
\end{table*}

\subsection{Complexity Analysis}
The size of an MILP problem is measured by the number of decision variables and the number of constraints.
In the proposed MILP encoding, the number of binary variables is bounded by $\mathcal{O}(|S|\cdot |A|)$, 
the number of real-valued variables is bounded by $\mathcal{O}(n \cdot |S|)$,
and the number of constraints is bounded by $\mathcal{O}(n \cdot |S| \cdot |A|)$.
MILP solvers work incrementally to synthesize a series of sound multi-strategies that are increasingly permissive.
Therefore, we may stop early to accept a sound (but not necessarily optimally permissive) multi-strategy if computational resources are limited. 

Prior to the MILP solution, we need to execute \agref{alg:bounds},
the most costly step of which is the computation of a Pareto optimal point in line 4. This is performed $|W|$ times,
where $|W|$ is exponential in the number of objectives $n$.
For each point, we compute a minimal weighted sum of expected total rewards for a given weight vector.
This is done using the value iteration-based method of~\cite{forejt2012pareto}. Value iteration does not have
a meaningful time complexity, but
is faster and more scalable than linear programming-based techniques in practice.

%\dave{should the theorem refer to optimal permissivity?}
%\red{I tried to add it in the theorem, see the blue text. But I felt like this may contradict with what we say about stop early in the complexity part. }

\section{Experimental Results}\label{sec:exp}

We have built a prototype implementation of the proposed approach, which uses the PRISM model checker~\cite{kwiatkowska2011prism} for computing Pareto optimal results of multi-objective synthesis in MDPs, 
and the Gurobi optimization toolbox~\cite{gurobi} for solving MILP problems. 
The experiments were run on a laptop with a 2.8 GHz Quad-Core Intel Core i7 CPU and 16 GB RAM.

\subsection{Case Studies} 
We applied our approach to three large case studies.
\footnote{Files are available from:
\href{https://www.prismmodelchecker.org/files/iccps22}{\url{https://www.prismmodelchecker.org/files/iccps22}}}. 
For each case study, we used two interval weight vectors representing preferences with different uncertainty levels.

The first case study is adapted from~\cite{feng2015controller}, which considers the control of an unmanned aerial vehicle (UAV) that interacts with a human operator for road network surveillance, with a varying model parameter to count the operator's workload and fatigue level that may lead to degraded mission performance. The controller synthesis aims to balance two objectives of mission completion time and risk, based on the specified uncertain human preferences. 

The second case study considers a task-graph scheduling problem inspired by~\cite{NPS13}. The controller synthesis aims to compute an optimal schedule for a set of dependent tasks based on human preferences of different processors, with a varying model parameter of the digital clock counter. 

The third case study models a team formation protocol~\cite{chen2011verifying} where a varying number of sensing agents cooperate to achieve certain tasks. The controller synthesis seeks to find an optimal schedule for these agents to meet the objectives of completing different tasks based on human preferences. 

\subsection{Results Analysis}
\tabref{tab:results} shows experimental results for these case studies. 
For each case study, we report the size of the MDP models in terms of the number of states and transitions, 
the size of the resulting MILP problems in terms of the number of decision variables (binary and real-valued) and constraints, the runtime for solving the MILP, 
and the number of permissive states (i.e., those with more than one allowed actions) in the synthesized controllers. 
We set a time-out of one hour for solving the MILP. 

Unsurprisingly, the size of MILP problems increases with larger MDP models. 
But the results demonstrate that our approach can scale to large case studies. 
For example, it takes less than one minute to solve the resulting MILP problem of ``uav 20'' model with 113,901 MDP states, which includes 119,897 binary variables, 455,604 real-valued variables, and 707,394 constraints in the MILP.
In most cases of ``uav'' and ``taskgraph'', a sound, optimally permissive multi-strategy is synthesized within one minute. 
However, the MILP solver failed to produce a feasible solution before time-out for some ``teamform'' cases, despite smaller MDP models than ``uav'' and ``taskgraph''. 
% In such cases, we stop the MILP solver at time-out to accept a feasible MILP solution (if it exists) that represents a sound but not necessarily optimally permissive multi-strategy.
In addition, we observed that increasing the uncertainty level of preferences (i.e., larger intervals) leads to synthesized controllers with larger numbers of permissive states. 

%For example, in the case of ``taskgraph 40'' with preferences $([0.8, 1], [0, 0.2])$, the synthesized controller does not have any permissive state, which means that the MILP solution yields an MDP strategy allowing at most one action in each state. However, when the same model switches to preferences $([0.1, 1], [0, 0.9])$ with a higher uncertainty level, the synthesized controller is a multi-strategy with $15,285$ permissive states.

\section{User Study}\label{sec:study}

We designed and conducted an online user study~\footnote{This user study has obtained the Institutional Review Board (IRB) approval.}
%from the University of Virginia (\#3829).} 
to evaluate the synthesized controllers.  
We describe the study design in \sectref{sec:design} and analyze the results in \sectref{sec:results}.

%========================================================
\subsection{Study Design} \label{sec:design}

\startpara{Participants} 
We recruited 100 individuals with a categorical age distribution of 6 (18-24); 58 (25-34); 28 (35-49); 6 (50-64); and 1 (65+) using Amazon Mechanical Turk (AMT). 
To ensure data quality, our study recruitment criteria required that participants must be able to read English fluently and had performed at least 50 tasks previously with an above 90\% approval rate on AMT. 
In addition, we injected attention check questions periodically during the study and rejected any response that failed attention checks.

\startpara{Procedure}
For each participant, we described the study purposes and asked them to consent to the study. After we asked about basic demographic information (e.g., age), the rest of the study consists of two phases: (i) eliciting human preferences, and (ii) evaluating the synthesized controllers.

First, we presented to each participant a grid map shown in \figref{fig:maze} and asked them to consider the planning problem for a robot to navigate from the start grid to the destination with three objectives: (1) minimizing the travel distance, (2) minimizing the risk encountered on route, and (3) maximizing the number of packages collected along the way. 
We used four different methods to elicit each participant's preferences over these objectives, including direct input of weight values (as illustrated in \figref{fig:elicit-direct}), Likert scaling (\figref{fig:elicit-scaling}), pairwise comparison of objective names (\figref{fig:elicit-name}), and pairwise comparison of optimal routes for individual objectives (\figref{fig:elicit-route}).
As described in \sectref{sec:preference}, we can derive a weight vector over objectives from the results of each preference elicitation method. Thus, by aggregating these four weight vectors resulting from different elicitation methods, we obtained an interval weight vector to represent each participant's preferences. 

\begin{figure}[tb]
    \centering
    \includegraphics[width=0.8\columnwidth]{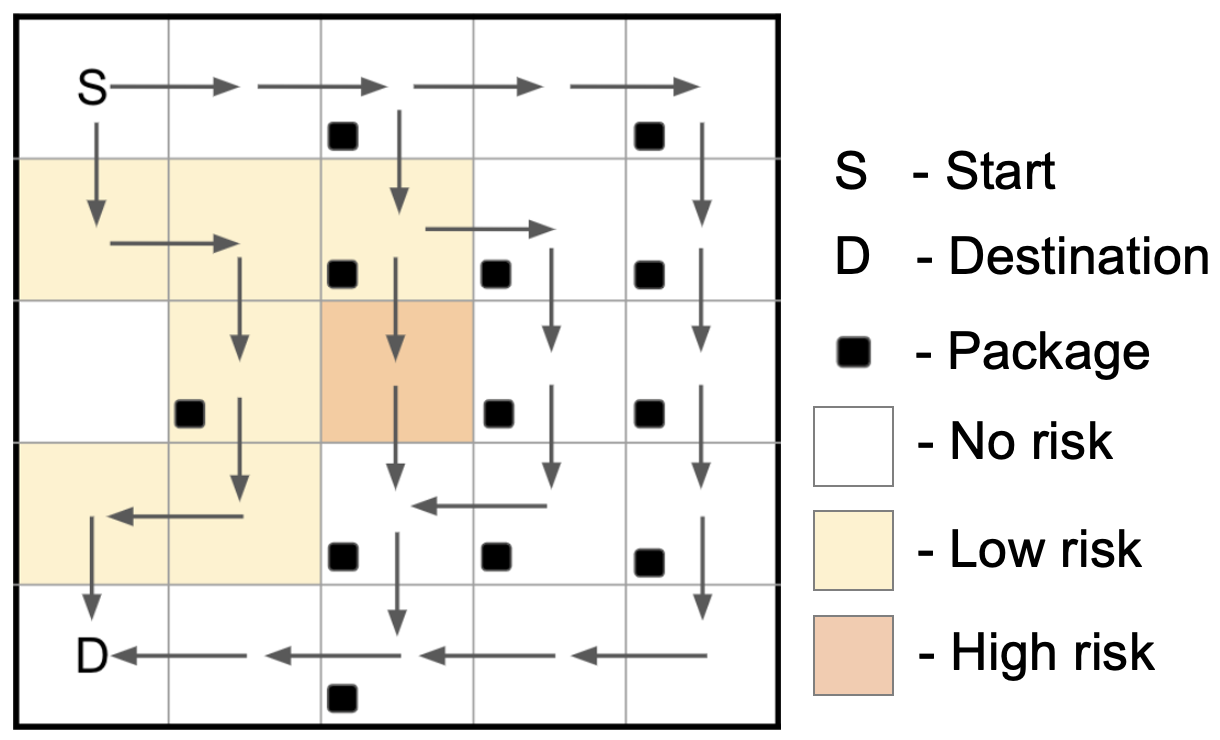}
    \caption{A grid map presented in the user study.}
    \label{fig:maze}
\end{figure}

\begin{figure}[tb]
    \centering
    \begin{subfigure}[b]{\columnwidth}
        \includegraphics[width=\columnwidth]{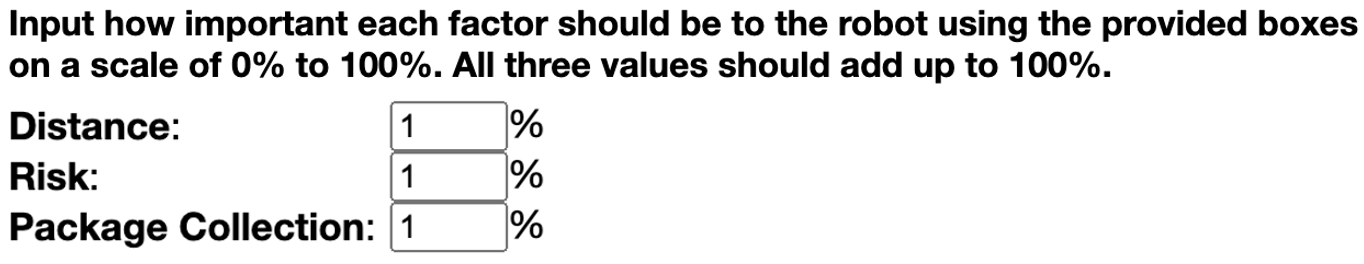}
        \caption{Direct input of weight values.}
        \label{fig:elicit-direct}
        \vspace{10pt}
    \end{subfigure}
    \begin{subfigure}[b]{\columnwidth}
        \includegraphics[width=\columnwidth]{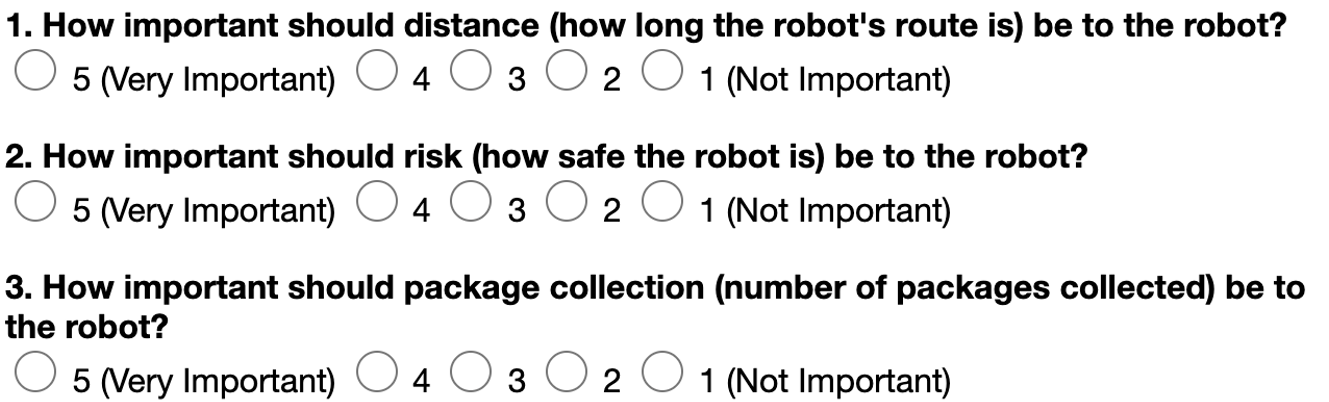}
        \caption{Likert scaling.}
        \label{fig:elicit-scaling}
        \vspace{10pt}
    \end{subfigure}
    \begin{subfigure}[b]{\columnwidth}
        \includegraphics[width=\columnwidth]{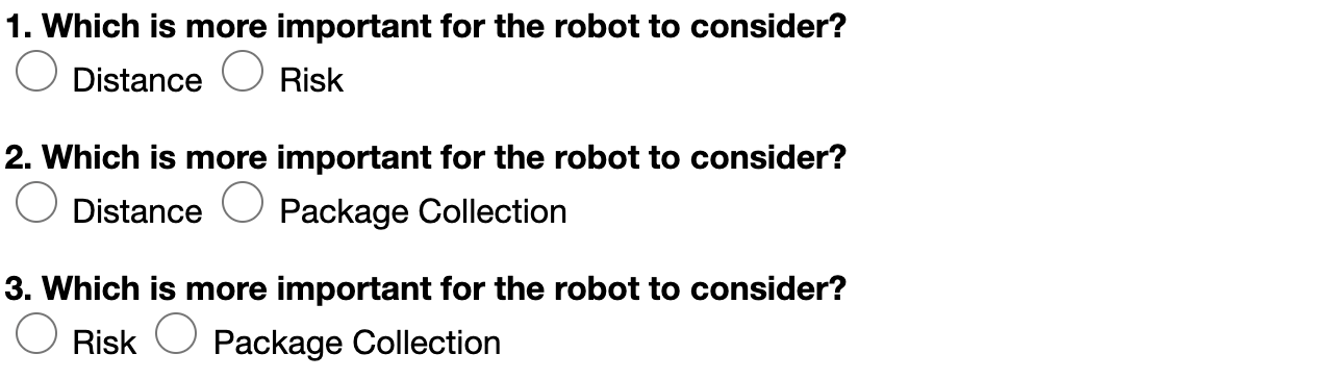}
        \caption{Pairwise comparison of objective names.}
        \label{fig:elicit-name}
        \vspace{10pt}
    \end{subfigure}
    \begin{subfigure}[b]{\columnwidth}
        \includegraphics[width=\columnwidth]{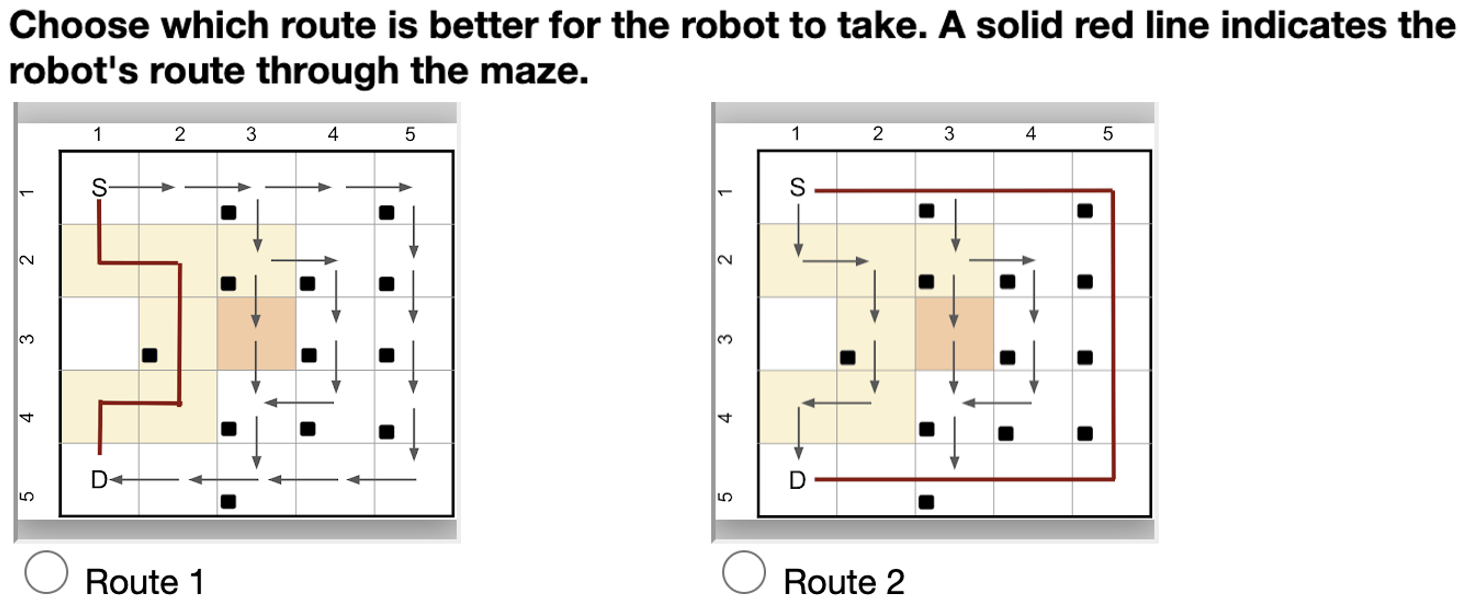}
        \caption{Pairwise comparison of routes that optimize individual objectives (e.g., route 1 for distance and route 2 for risk).}
        \label{fig:elicit-route}
    \end{subfigure}
    \caption{Four different methods for eliciting preferences.}
    \label{fig:elicit}
\end{figure}

Next, based on the elicited human preferences, we applied the proposed approach to synthesize optimal robotic controllers for three different grid maps (including \figref{fig:maze} and two other similar maps). 
We randomized the order of maps for different particiants to counterbalance the ordering confound effect. 
For each map, we asked participants a set of questions to evaluate the synthesized controllers. 
We describe the evaluation design, including manipulated factors, dependent measures, and hypotheses as follows.

\startpara{Manipulated factors and dependent measures} 
We performed a within-subject experiment in which all participants were exposed to all evaluation conditions. We manipulated two independent factors: preferences and permissivity for the controller synthesis.
For each map, we first presented a pair of MDP strategies (visualized as plans in the grid map) side by side:
one is a sound strategy synthesized based on the elicited preferences,
and the other is an arbitrary strategy unsound for preferences.  
\figref{fig:comp1} shows the list of evaluation questions. 
We asked participants about their satisfaction and perceived accuracy of each plan. 
We also asked them to choose which plan they preferred. 

Then, we presented side by side a strategy (visualized as a single route plan) and a multi-strategy (visualized as a possible multiple route plan), which are both synthesized based on the elicited preferences but with different degrees of permissivity. 
We asked participants to compare the synthesized strategy and multi-strategy in terms of 
favor (``Which route do you like better?''),
informativity (``Which route provides more information?''),
and satisfaction (``Which route are you more satisfied with?'').
The exact questionnaire can be found in \figref{fig:comp2}.

\startpara{Hypotheses}
We made the following hypotheses based on the two manipulated factors. 

Comparing strategies synthesized based on the elicited preferences and arbitrary strategies:
\begin{itemize}
    \item \textbf{H1:} Preference-based strategies are more favorable than unsound arbitrary strategies.
    \item \textbf{H2:} Preference-based strategies are perceived as more accurate than unsound arbitrary strategies.
    \item \textbf{H3:} Preference-based strategies yield better satisfaction than unsound arbitrary strategies.
\end{itemize}

Comparing strategies and multi-strategies synthesized based on the elicited preferences:
\begin{itemize}
    \item \textbf{H4:} Multi-strategies are more favorable than strategies.
    \item \textbf{H5:} Multi-strategies are perceived as more informative than strategies.
    \item \textbf{H6:} Multi-strategies yield better satisfaction than strategies.
\end{itemize}

\begin{figure}[t]
    \centering
    \includegraphics[width=0.85\columnwidth]{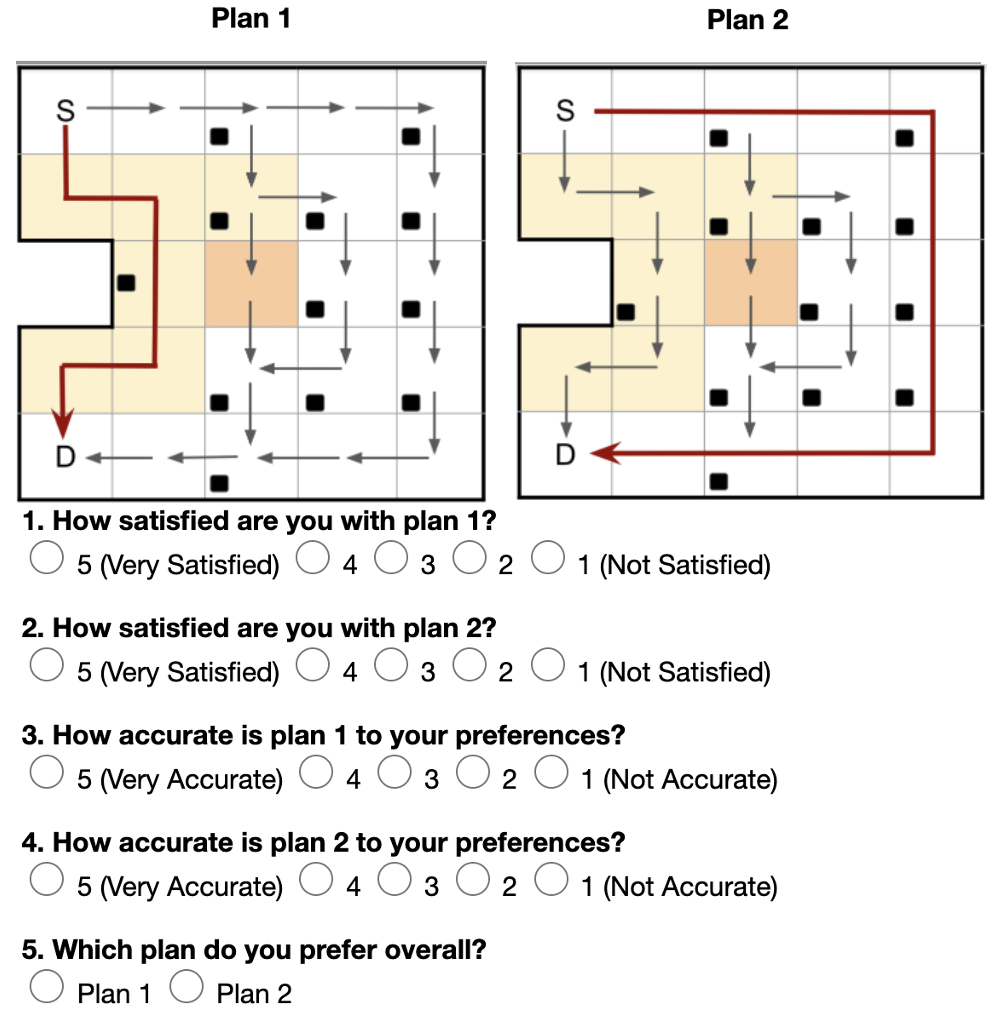}
    \caption{Evaluation of a synthesized strategy compared to an arbitrary strategy. Users were told these were possible robotic plans generated based on their input preferences, but not which plan was actually arbitrary.}
    \label{fig:comp1}
\end{figure}

\begin{figure}[t]
    \centering
    \includegraphics[width=0.8\columnwidth]{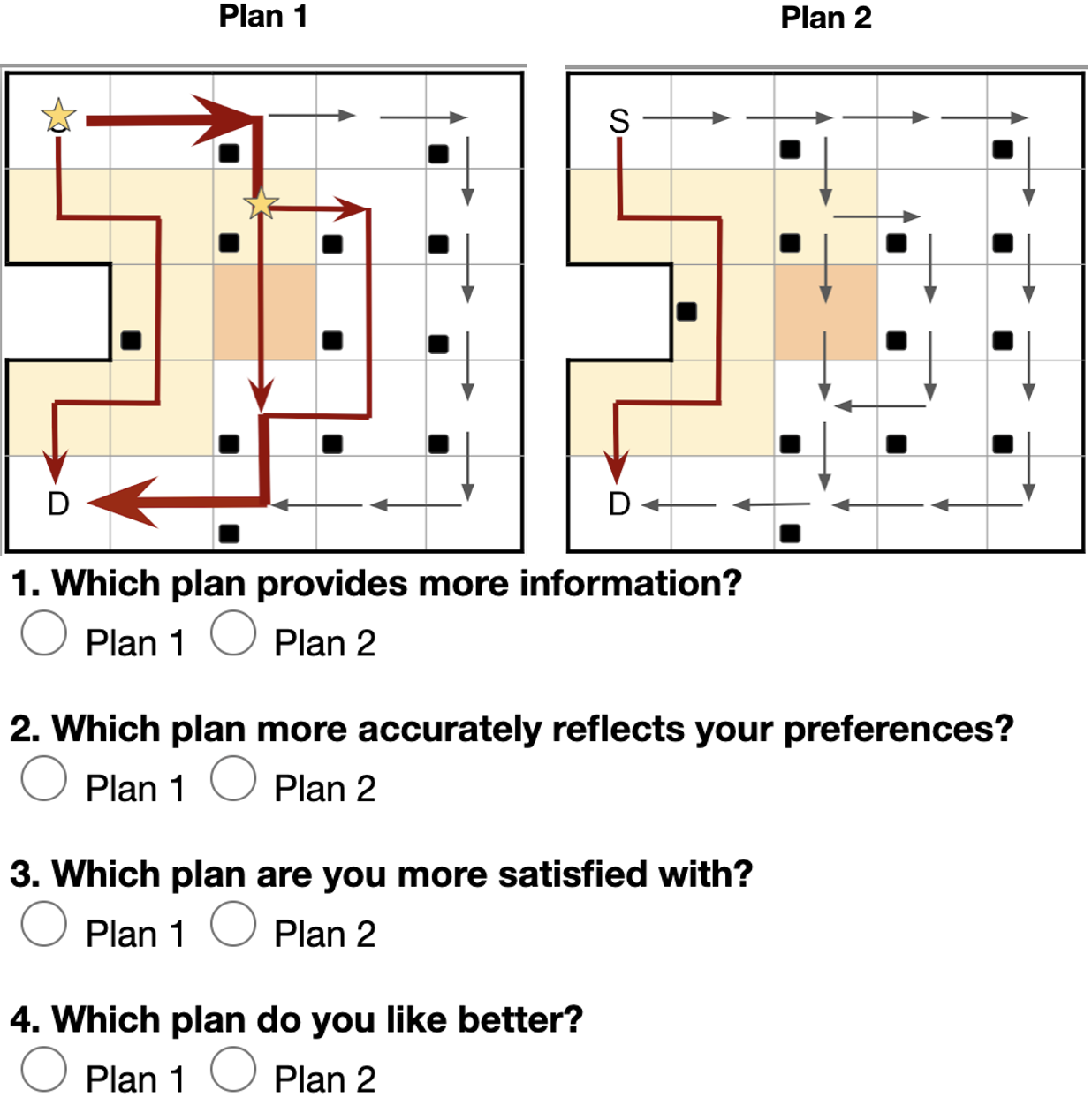}
    \caption{Evaluation comparison of a synthesized multi-strategy (plan 1) and a synthesized strategy (plan 2). Users were told these were possible robotic plans generated based on their input preferences. Stars indicate permissive states with multiple allowed actions.}
    \label{fig:comp2}
\end{figure}

%========================================================
\subsection{Results Analysis} \label{sec:results}

\begin{figure}[t]
    \centering
    \includegraphics[width=0.4\textwidth]{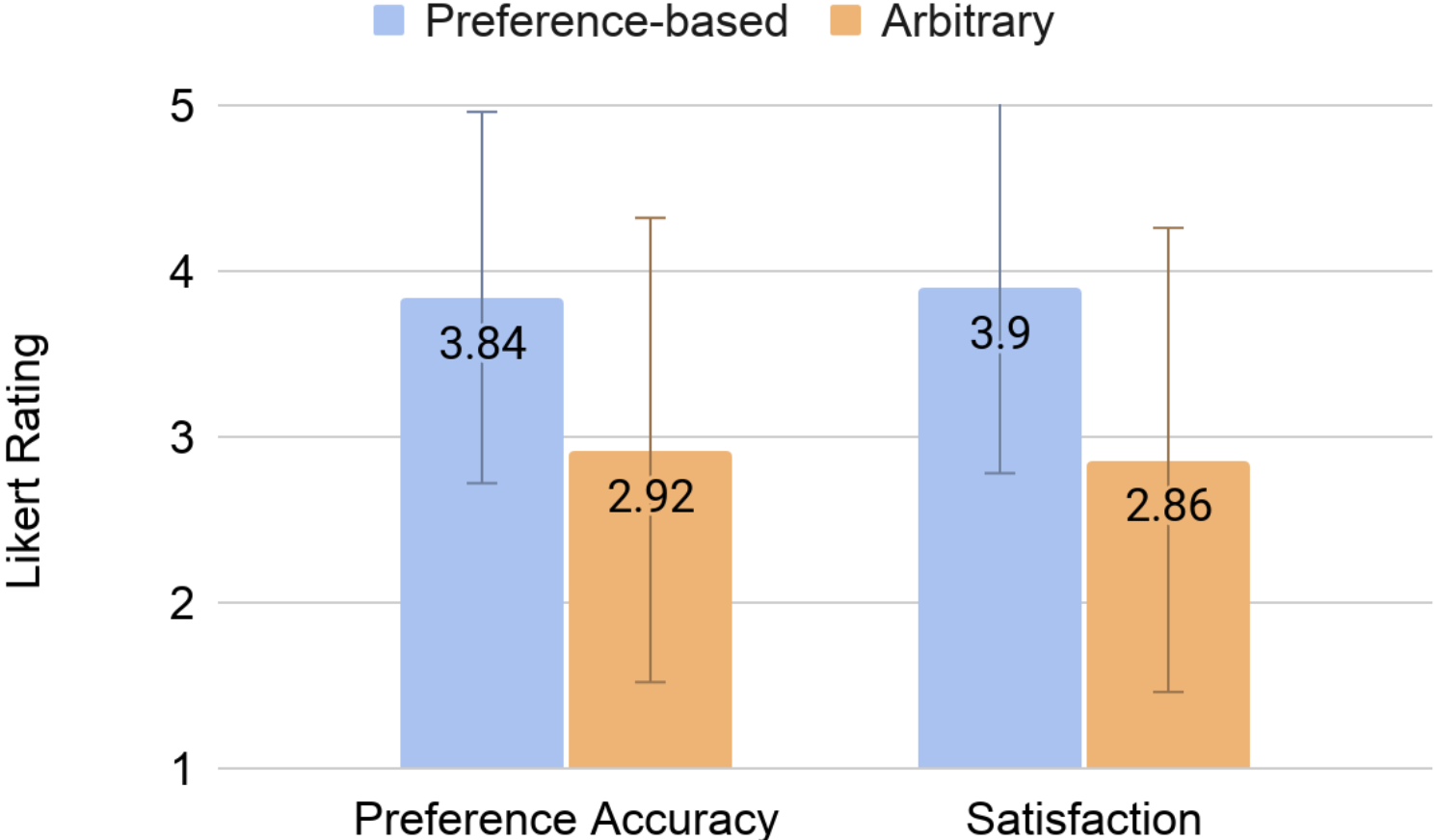}
    \caption{Mean and standard deviation of 5-point Likert ratings about perceived preference accuracy and user satisfaction for preference-based and arbitrary strategies.}
    \label{fig:comp1-results}
\end{figure}

\startpara{Comparing preference-based and arbitrary strategies} 
To evaluate hypothesis H1, we utilize a chi-squared test~\cite{seltman2012experimental} to prove the statistical significance in the frequency of strategy selection, assuming an expected frequency of 50/50 to represent a random selection of strategies by users. We use an alpha value of 0.05 and thus retain a confidence level of 95\% for our hypotheses. We assume a null hypothesis that the user selection of strategies will be random.
We find that users favor preference-based strategies about 63$\%$ of the time overall ($\chi^2$:$\alpha$= 0.05, $\chi^2$ = 21.33, CritVal = 3.84, p$\leq$0.00001, Significant.); 
they choose preference-based strategies over arbitrary strategies more often for all three maps (71\%, 59\%, 60\%).
\emph{Thus, the data supports H1.}

To evaluate hypotheses H2 and H3, shown in Figure \ref{fig:comp1-results} we employ one-way repeated measures ANOVA tests~\cite{seltman2012experimental} to prove the statistical significance of the mean of all responses between preference-based strategies and arbitrary strategies. We use an alpha value of 0.05 and assume a null hypothesis that users will perceive preference accuracy and be satisfied with both strategies at a similar rate.
We find that users rated preference-based strategies as significantly more accurate to their objective preferences (rANOVA:$\alpha$= 0.05, F(1,598) = 74.71, p$\leq$0.00001, Significant.). 
\figref{fig:comp1-results} also shows that users were significantly more satisfied with preference-based strategies than another arbitrary strategy through the plan (rANOVA:$\alpha$= 0.05, F(1,598) = 105.28, p$\leq$0.00001, Significant.). 
\emph{Thus, the data supports H2 and H3.}

%Thus, showing it is more important to select a route that fulfills the user’s wants over one that is arbitrary, even if optimal in achieving the potential goal. So, we recommend the use of preference-based strategies over potentially arbitrary strategies in single-path route planning situations. These findings are most likely due to a perceived increase in user agency (i.e., the ability of users to make choices and act independently, \citet{emirbayer_mische_1998}) since the users were able to choose the aspects of the route themselves and thus fulfill their own wants and needs. 

\begin{figure}[t]
    \centering
    \includegraphics[width=0.4\textwidth]{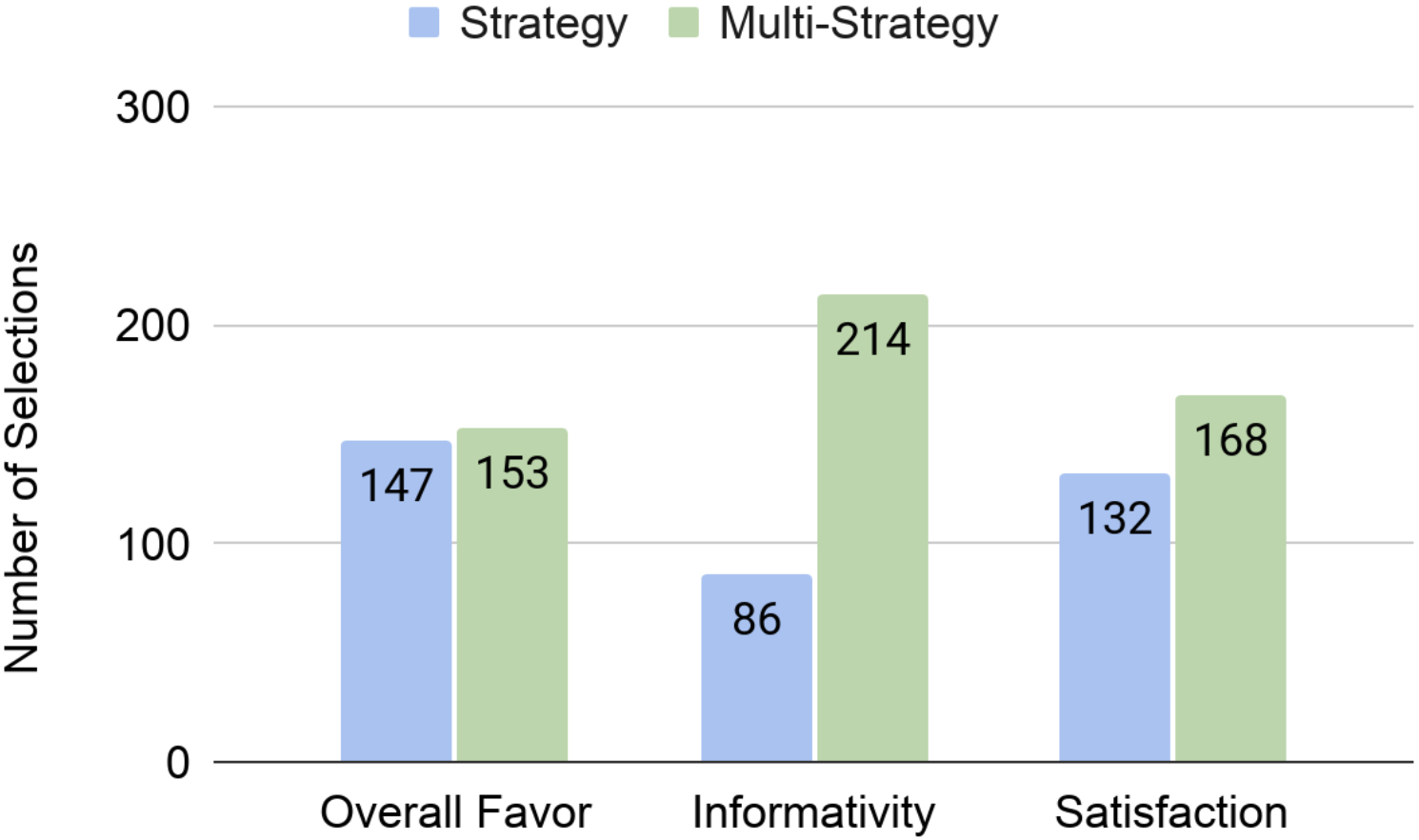}
    \caption{Pairwise comparison of the synthesized strategies and multi-strategies regarding overall favor, informativity, and satisfaction.}
    \label{fig:comp2-results}
\end{figure}

\startpara{Comparing strategies and multi-strategies}
We use chi-squared tests with an expected frequency of 50/50 and an alpha value of 0.05 to evaluate hypotheses H4, H5, and H6 with the study results shown in \figref{fig:comp2-results}.

Column 1 (Overall Favor) of \figref{fig:comp2-results} shows that users do not significantly favor multi-strategies over less permissive strategies ($\chi^2$:$\alpha$= 0.05, $\chi^2$ = 0.12, CritVal = 3.84, p$\leq$0.729, Not Significant.), only slightly favoring multi-strategies to a single strategy counterpart.
\emph{Thus, the data rejects H4.}

Column 2 (Informativity) of \figref{fig:comp2-results} shows users agreed about 71$\%$ of the time that multi-strategies provided them more information ($\chi^2$:$\alpha$= 0.05, $\chi^2$ = 54.61, CritVal = 3.84, p$\leq$0.00001, Significant.).  
\emph{Thus, the data supports H5.}

Column 3 (Satisfaction) of \figref{fig:comp2-results} shows users were more satisfied with multi-strategies 56$\%$ of the time ($\chi^2$:$\alpha$= 0.05, $\chi^2$ = 4.32, CritVal = 3.84, p$\leq$0.038, Significant.). 
\emph{Thus, the data supports H6.}

\vspace{10pt}
\startpara{Summary}
We accept all hypotheses except H4 based on the statistical analysis. 
The user study results show that it is beneficial to synthesize strategies that account for human preferences. 
In addition, multi-strategies are more informative and yield better user satisfaction. 
However, sometimes less is more, participants do not always favor multi-strategies over strategies that are simpler to understand. 

%These findings are most likely due to the multi-strategy increasing the user’s understanding of the map and its routes. Users are better informed about which routes will meet their wants and all the possible movements that the robot may make instead of only one possible route. The increase in information gives the user more control and thus makes them more satisfied with the strategy’s performance. Since the goal of the multi-strategy is to provide the user will all possible options and thus increase the user’s information level, we can say that the strategy has fulfilled its purpose and satisfied the user in this regard. We conclude that multi-strategies are a more usable choice for route planning situations utilizing user preferences than a single route strategy, even if users do not necessarily favor them. 

\section{Related Work}\label{sec:related}

\startpara{Human preferences}
Mathematical models of human preferences have been studied broadly in the field of social choice theory~\cite{arrow2012social}.
There are many different representations of human preferences, for example, encoded as reward functions for robot trajectory planning~\cite{sadigh2017active} and deep reinforcement learning~\cite{christiano2017deep},
or specified using temporal logics~\cite{li2015preference,mehdipour2020specifying}.
In the context of multi-objective optimization~\cite{marler2004survey}, preferences are represented as weights indicating the relative importance of objectives. 
Optimization methods can vary depending on when and how humans articulate their preferences.
Humans can indicate their preferences \emph{a priori} before running the optimization algorithm, they can progressively provide input during the optimization process, 
or they can select \emph{a posteriori} a solution point from a set of Pareto optimal results.
Our work considers \emph{a priori} elicitation of human preferences represented as weights for multiple objectives.

%---------------------------------------------------------

\startpara{Multi-objective controller synthesis for MDPs}
Multi-objective optimization has been well-studied in operation research and engineering~\cite{roy1981multicriteria,marler2004survey}.
In recent years, multi-objective optimization for MDPs has been considered from a formal methods perspective~\cite{chatterjee2006markov,etessami2007multi,forejt2011quantitative,forejt2012pareto,HJKQ20,DKQR20},
which presents theories and algorithms for verifying multi-objective properties, synthesizing strategies, and approximating Pareto curves.
More recently, such techniques have been applied to multi-objective robot path planning~\cite{lacerda2017multi} and multi-objective controller synthesis for autonomous systems that account for human operators' workload and fatigue levels~\cite{feng2015controller}. 
However, existing work does not account for the uncertainty of human preferences in the relative importance of objectives. 

There is a line of work (e.g., \cite{wolff2012robust,hahn2017multi,cubuktepe2020scenario}) considering uncertain MDPs where transition probabilities and rewards are represented as an uncertain set of parameters or intervals.
Our work is different in the sense that we consider the uncertainty in human preferences of different objectives.

Our proposed approach is based on mixed-integer linear programming (MILP). There exist several MILP-based solutions to compute counterexamples and witnesses for MDPs~\cite{wimmer2014minimal, feng2018counterexamples,feng2016human,funke2020farkas}. However, these methods are not directly applicable for controller synthesis which is a different problem. The most relevant work is~\cite{DFK+15} that presents an MILP-based method for permissive controller synthesis of probabilistic systems. As discussed in \sectref{sec:intro}, our approach is inspired by~\cite{DFK+15} but has several key differences (e.g., \cite{DFK+15} does not consider the controller soundness with respect to multi-objective properties and human preferences).

\section{Conclusion}\label{sec:conclusion}

%\red{add some discussion about the broder implications of our work, answering the so what question...}

In this paper, we developed a novel approach that accounts for uncertain human preferences in the multi-objective controller synthesis for MDPs.
The proposed MILP-based approach synthesizes a sound, optimally permissive multi-strategy with respect to a multi-objective property and an uncertain set of human preferences.
We implemented and evaluated the proposed approach on three large case studies. 
Experimental results show that our approach can be successfully applied to synthesize sound, optimally permissive multi-strategies with varying MDP model size and uncertainty level of human preferences.
In addition, we designed and conducted an online user study with 100 participants using Amazon Mechanical Turk, which shows statistically significant results about user satisfaction of the synthesized controllers.  

There are several directions to explore for possible future work. 
First, we will extend our approach for a richer set of multi-objective properties beyond expected total rewards, such as the temporal logic-based multi-objective properties considered in~\cite{forejt2011quantitative}.
Second, we will extend our approach for a variety of probabilistic models beyond MDPs, such as stochastic games and POMDPs. 
Last but not least, we will apply our approach to a wider range of real-world CPS applications (e.g., autonomous driving, smart cities).

\begin{acks}
This work was supported in part by National Science Foundation grants CCF-1942836 and CNS-1755784, and European Research Council under the European Union’s Horizon 2020 research and innovation programme (grant agreement No. 834115, FUN2MODEL).
Any opinions, findings, and conclusions or recommendations expressed in this material are those of the author(s) and do not necessarily reflect the views of the grant sponsors.
\end{acks}

\bibliographystyle{ACM-Reference-Format}
\bibliography{references}


\newpage

\appendix
\section{Proofs}\label{sec:proof}

Here, we prove the correctness of our MILP encoding, as stated in \thmref{thm:milp}. This result adapts and extends the proof for the MILP encoding given in~\cite{DFK+15}
(specifically the case for what are called \emph{static} penalty schemes and deterministic multi-strategies).
We require the following auxiliary lemma,
where $\Emss(r_i)$ denotes the expected total reward for reward structure $r_i$ under strategy $\sigma$ of MDP $\cM$, from a particular starting state $s$.

\vskip6pt
\begin{lemma}\label{lem:mdp}
Let $\cM=(S, s_0, A, \delta)$ be an MDP, $\phi=([r_1]_{\min},$
$\dots,[r_n]_{\min})$ be a multi-objective property,
and $\theta$ be a multi-strategy. Consider the inequalities for $s \in S$, $1\leq i\leq n$:
$$ \mu_{i,s} \le \min_{a \in \theta(s)}\sum_{t\in S} \delta(s,a)(t) \cdot \mu_{i,t} + r_i(s,a) $$
$$ \nu_{i,s} \ge \max_{a \in \theta(s)}\sum_{t\in S} \delta(s,a)(t) \cdot \mu_{i,t} + r_i(s,a) $$
Then values $\hat{\mu}_{i,s},\hat{\nu}_{i,s}\in\mathbb{R}$, for $s \in S$ and $1\leq i\leq n$,
are a solution to the above inequalities if an only if
$\hat{\mu}_{i,s} = \inf_{\sigma \lhd \theta} \Emss(r_i)$ and 
$\hat{\nu}_{i,s} = \sup_{\sigma \lhd \theta} \Emss(r_i)$.

%Then the following statements are true:
% \begin{itemize}
%     \item[(a)] $\hat{\mu}_{i,s} = \inf_{\sigma \lhd \theta} \Emss(r_i)$ and 
%         $\hat{\nu}_{i,s} = \sup_{\sigma \lhd \theta} \Emss(r_i)$
%          is a solution to the above inequalities.
%     \item[(b)] \blue{A solution $\hat{\mu}_{i,s}$ and $\hat{\nu}_{i,s}$ to the above inequalities satisfies
%         $\hat{\mu}_{i,s} \le \inf_{\sigma \lhd \theta} \Emss(r_i)$
%         and $\hat{\nu}_{i,s} \ge \sup_{\sigma \lhd \theta} \Emss(r_i)$
%         for all $s$ whenever the following condition holds: 
%         for every $s$ with $\hat{\mu}_{i,s} > 0$, $\hat{\nu}_{i,s} > 0$, every $\sigma \lhd \theta$, 
%         there is a finite path $\pi=s, \dots, s_m$ starting in $s$ and ending in $s_m$
%         that satisfies $\Prmss >0$ and $r_i(s_m,a_m)>0$. }
% \end{itemize}
\end{lemma}
\begin{proof}
The above follows from standard results on the solution of MDPs~\cite{puterman1994markov},
noting that there is a separate set of inequalities for each $\mu_i$ and $\nu_i$ and $1\leq i\leq n$.
This also relies on our assumption (see Section~\ref{sec:problem}) that
a designated set of zero-reward end states is always reached with probability 1,
ensuring that expected total rewards are finite and removing the need to deal with zero-reward loops
(whereas \cite{DFK+15} deals with the latter through
additional MILP variables and constraints).
%
% We prove this lemma by applying results from the MDP theory. 
% Claim (a) follows~\cite[Theorem 7.1.3]{puterman1994markov}, which gives Bellman equations as a characterisation of values in MDPs. The inequalities in this lemma are just a relaxation of Bellman equations. 
% Claim (b) \red{how to adapt the proof of Lemma 4.1 in LMCS paper for this? Also based on ATVA paper, we need to impose reward-finiteness restrictions...how does it impact this proof?}
\end{proof}

\vskip6pt
\correctness*
 
%==================================== 
\begin{proof}
We prove that:
(1) every multi-strategy $\theta$ induces a satisfying assignment to the MILP
such that the permissive penalty $\lambda(\theta)=\sum_{s\in S} \sum_{a\in \alpha(s)} (1-\eta_{s,a})$, 
and (2) vice versa. 

\startpara{Direction (1)}
We start by proving that, given a sound multi-strategy $\theta$, 
we can construct a satisfying assignment \\
$\{\hat{\eta}_{s,a}, \hat{\mu}_{i,s}, \hat{\nu}_{i,s}\}_{s \in S, a \in A, 1\leq i\leq n}$
to the MILP constraints.
For $s\in S$ and $a \in \alpha(s)$, we set $\hat{\eta}_{s,a}=1$ if 
$s$ is a reachable state under $\theta$ and $a \in \theta(s)$;
otherwise, we set $\hat{\eta}_{s,a}=0$.

Thus, the permissive penalty $\lambda(\theta)$ that counts the total number of disallowed actions in reachable states under $\theta$ equals to $\sum_{s\in S} \sum_{a\in \alpha(s)} (1-\hat{\eta}_{s,a})$.
Constraints (\ref{eq:c1}) and (\ref{eq:c2}) are satisfied for all unreachable states, because both sides of the inequalities are zero. 
For reachable states, constraint (\ref{eq:c1}) is trivially satisfied if the scaling factor $c$ is large enough;
constraint (\ref{eq:c2}) is also satisfied, because a reachable state under strategy $\theta$ should have at least one allowed action.

We set $\hat{\mu}_{i,s} {=} \inf_{\sigma \lhd \theta} \Emss(r_i)$
and $\hat{\nu}_{i,s} {=} \sup_{\sigma \lhd \theta} \Emss(r_i)$.
Constraints (\ref{eq:c3}) and (\ref{eq:c4}) are satisfied 
for $a \in \theta(s)$ thanks to \lemref{lem:mdp}.
If $a \not \in \theta(s)$, then (\ref{eq:c3}) and (\ref{eq:c4}) are also trivially satisfied because $\hat{\eta}_{s,a}=0$. 
By the soundness definition, we have 
$\hat{\mu}_{i,s_0} \ge \underline{b}_i$ and $\hat{\nu}_{i,s_0} \le \overline{b}_i$.
This gives the satisfaction of constraints (\ref{eq:c5}) and (\ref{eq:c6}). 
%\red{this may need to connect with the correctness of algorithm 1}

\startpara{Direction (2)}
Given a satisfying MILP assignment\\
$\{\hat{\eta}_{s,a}, \hat{\mu}_{i,s}, \hat{\nu}_{i,s}\}_{s \in S, a \in A, 1\leq i\leq n}$,
we construct $\theta$ for $\cM$ by putting $\theta(s)=\{a \in \alpha(s) | \ \hat{\eta}_{s,a}=1\}$ for all $s \in S$.
Thanks to constraints (\ref{eq:c3}) and (\ref{eq:c4}), and \lemref{lem:mdp},
we have that 
$\hat{\mu}_{i,s} {=} \inf_{\sigma \lhd \theta} \Emss(r_i)$
and $\hat{\nu}_{i,s} {=} \sup_{\sigma \lhd \theta} \Emss(r_i)$
for each $1\leq i \leq n$.
Using also constraints (\ref{eq:c5}) and (\ref{eq:c6}),
we have that, for any strategy $\sigma\lhd\theta$,
$\Emsi(r_i) \geq \hat{\mu}_{i,s_0} \geq \underline{b}_i$
and $\Emsi(r_i) \leq \hat{\nu}_{i,s_0} \leq \overline{b}_i$;
thus the multi-strategy $\theta$ is sound with respect to $\phi$ and $\vb*{\tilde{w}}$.
As in the reverse direction above,
the permissiveness of $\theta$ is 
$\lambda(\theta)=\sum_{s\in S} \sum_{a\in \alpha(s)} (1-\hat{\eta}_{s,a})$,
taken from the objective function of the MILP.
\end{proof}

\end{document}